\documentclass{article}
\usepackage{arxiv}
\usepackage[utf8]{inputenc} 
\usepackage[T1]{fontenc}    
\usepackage{hyperref}       
\usepackage{url}            
\hypersetup{urlcolor=black}
\usepackage{booktabs}       
\usepackage{amsfonts}       
\usepackage{amssymb,amsthm}
\usepackage{float}
\usepackage{nicefrac}       
\usepackage{microtype}      
\usepackage{graphicx}
\usepackage{tabularx}
\usepackage[linesnumbered,ruled,vlined]{algorithm2e}
\usepackage{listings}
\usepackage{ifpdf}
\usepackage{keyval}
\usepackage{iflang}

\newtheorem{theorem}{Theorem}

\def\dproblem#1#2#3{
\begin{flushleft} 
  \noindent 
  {\sc #1}\\
  {\bf Instance: }#2.\\
  {\bf Question: }#3?\\
\end{flushleft}}

\def\fproblem#1#2#3{
\begin{flushleft} 
  \noindent 
  {\sc #1}\\
  {\bf Instance: }#2.\\
  {\bf Output: }#3?\\
\end{flushleft}}

\newcommand{\NwFlow}{{\sc Max Network Flow}}
\newcommand{\TNwFlow}{{\sc Threshold Network Flow with Reservoirs}}
\newcommand{\var}{\texttt}

\title{A spiking neural algorithm\\for the Network Flow problem}

\author{
  Abdullahi Ali \\
  School for Artificial Intelligence\\
  Radboud University\\
  Nijmegen, The Netherlands \\
  \texttt{a.ali@student.ru.nl} \\
   \And
 Johan Kwisthout \\
  Donders Institute for Brain, Cognition, and Behaviour\\
  Radboud University\\
  Nijmegen, The Netherlands \\
  \texttt{j.kwisthout@donders.ru.nl} \\
}

\begin{document}
\maketitle

\begin{abstract}
    It is currently not clear what the potential is of neuromorphic hardware beyond machine learning and neuroscience. In this project, a problem is investigated that is inherently difficult to fully implement in neuromorphic hardware by introducing a new machine model in which a conventional Turing machine and neuromorphic oracle work together to solve such types of problems. 
We show that the $\mathsf{P}$-complete \NwFlow\ problem is intractable in models where the oracle may be consulted only once (`create-and-run' model) but becomes tractable using an interactive (`neuromorphic co-processor') model of computation. More in specific we show that a logspace-constrained Turing machine with access to an interactive neuromorphic oracle with linear space, time, and energy constraints can solve \NwFlow. A modified variant of this algorithm is implemented on the Intel Loihi chip; a neuromorphic manycore processor developed by Intel Labs. We show that by off-loading the search for augmenting paths to the neuromorphic processor we can get energy efficiency gains, while not sacrificing runtime resources. This result demonstrates how $\mathsf{P}$-complete problems can be mapped on neuromorphic architectures in a theoretically and potentially practically efficient manner.
\end{abstract}


\section{Introduction}

Neuromorphic computing has been one of the proposed novel architectures to replace the von Neumann architecture that has dominated computing for the last 70 years \cite{mead1990neuromorphic}. These systems consist of low power, intrinsically parallel architectures of simple spiking processing units.  In recent years numerous neuromorphic hardware architectures have emerged with different architectural design choices \cite{akopyan2015truenorth,davies2018loihi,schemmel2012live,khan2008spinnaker}.

It is not exactly clear what the capabilities of these neuromorphic architectures are, but several properties of neuromorphic hardware and application areas have been identified in which neuromorphic solutions might yield efficiency gains compared to conventional hardware architectures such as CPUs and GPUs \cite{akopyan2015truenorth,davies2018loihi,khan2008spinnaker,schemmel2012live}. These applications are typically inherently event-based, easy to parallelize, are limited in terms of their energy budget and can be implemented on a sparse communication architecture where processors can communicate with small packets of information.

A potential major application area of these neuromorphic architectures is machine learning. This is motivated by the results deep neural networks have achieved in machine learning \cite{lecun2015deep}, where these loosely brain-inspired algorithms have dramatically redefined machine learning and pattern recognition. However, deep neural networks tend to consume a significant amount of energy. This energy bottleneck is one of the major reasons why these deep networks have not been successfully employed in embedded AI applications such as robotics. Neuromorphic processors, on the other hand, could potentially solve this bottleneck and fuel a new leap forward in brain-inspired computing solutions for AI.

There are several other areas that could greatly benefit from energy efficiency. One of these applications is numerical algorithms in scientific computing \cite{severa2016spiking}.  Traditionally, neural networks are trained by automatically modifying their connection weights until a satisfactory performance is achieved. Despite its success in machine learning, this approach is not suitable for scientific computing or similar areas since it may require many training iterations and does not produce precise results.

Alternatively, we can abandon learning methods and design the networks by hand. One way to do this is to carefully construct a network of (non-stochastic) spiking neurons to encode information in the spiking patterns or spike-timing.  One can, for example, introduce a synaptic delay to encode distance or introduce a spiking clock mechanism to encode values using the spike-time difference of a readout neuron and a clock.

 Efforts have been undergone on designing neural algorithms for primitive operations \cite{verzi2017optimization,severa2016spiking} and relatively straightforward computational problems \cite{maass2014noise, jonke2016solving}, but it is not clear how these methods can be scaled up to more complex algorithms. In this work, we build upon the algorithm design approach advocated by \cite{verzi2017optimization,severa2016spiking} and propose a systematic way to analyse and expand the potential application space of neuromorphic hardware beyond machine learning. In this light, we look at a problem of non-trivial complexity: the maximum flow problem.
 
 The input of the maximum flow problem consists of a weighted graph, where each weight denotes the capacity of a certain edge. We have two special nodes: a source and a sink. The source is the starting point of the graph and \textit{produces} flow and the sink is the terminal point of the graph and \textit{consumes} flow. The objective is to push as much flow as possible from source to sink while respecting the capacity constraints for each edge. The canonical method to solve this problem is the Ford-Fulkerson method \cite{ford1955simple}. In this method, one repeatedly searches for augmenting paths. These are simple paths from source to sink through which we can still push flow (i.e. no edge on the path has reached full capacity). If such a path is found, we determine the minimum capacity edge on this path and increment the flow through each edge in this path with this minimum capacity. This process is repeated until all augmenting paths have been found. In figure \ref{flownet} you can see an example of such a flow network.
 
 \begin{figure}[h]
     \centering
     \includegraphics[width=0.4\textwidth]{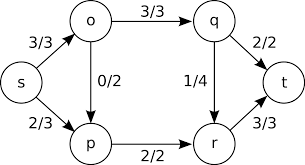}
     \caption{\textit{Example of a flow network with source 's' and sink 't'. Note that each has two numbers f/c associated to it. 'f' indicates the current flow through the edge and 'c' indicates the total capacity of the edge. }}
     \label{flownet}
 \end{figure}
 
 The maximum flow problem arises in many areas, such as logistics, computer networking, operations research and bioinformatics. With the huge increase in data in these application areas, the flow networks will similarly increase in size, demanding a need for faster and parallel algorithms.\\ Unfortunately, network flow algorithms are difficult to parallelize. The vast majority of network flow algorithms are implementations of the Ford-Fulkerson method. Augmenting paths have to be found from source to sink and flow has to be pushed through these paths, requiring fine-grained locking of the edges and nodes on the path which introduces expensive overhead.
 
 Several parallel implementations for finding the maximum flow in a network exist \cite{anderson1995parallel,bader2006cache,hong2011asynchronous}, but they often are carefully optimised against contemporary computer architecture or do not offer significant performance advantages over optimised serialised solutions. As a matter of fact, theoretical results show that Network flow is a $\mathsf{P}$-complete problem \cite{goldschlager1982maximum}. This means that it probably cannot be efficiently solved on a parallel computer\footnote{A decision problem $D$ belongs to the class $\mathsf{NC}$ (class of decision problems that can be decided in polylogarithmic time on a parallel computer) if there exist constants $c$ and $k$ such that $D$ can be decided in $log(n^c)$ time using only $k$ processors \cite{arora2009computational}.}. It is therefore likely that it cannot be fully implemented in neuromorphic hardware. Another corollary is that $\mathsf{P}$-complete problems also have the property that they likely cannot be solved with a guaranteed logarithmic space bound (henceforth denoted as logspace constraint). The class of polynomial-time solvable problems that respect the logspace constraint is called $\mathsf{L}$ and it is likely that no $\mathsf{P}$-complete problem is in this class.  
 
 An interesting direction of investigation is whether we can achieve this logspace constraint by utilizing aspects of both conventional and neuromorphic hardware. In this light, we introduce a new machine model in which a conventional computing device can consult a neuromorphic oracle to offload certain computations. Under this machine model, we introduce a lattice of complexity classes: $\mathsf{C}^{\mathcal{S}(R_T,R_S)}$, in which the conventional computing device can construct an oracle query (using resources $R_T = (\mathrm{TIME}, \mathrm{SPACE})$ and then consult a neuromorphic oracle (defined by resources $R_S = (\mathrm{TIME}, \mathrm{SPACE}, \mathrm{ENERGY})$. Importantly, in addition to the more traditional resources time and space we also take energy (defined as the number of spikes of the neuromorphic device) into account. We show that the $\mathsf{P}$-complete \textsc{Max Network Flow} problem is in $\mathsf{L}^{\mathcal{S}((\mathcal{O}(n^c),\mathcal{O}(\log n)),(\mathcal{O}(n),\mathcal{O}(n),\mathcal{O}(n)))}$ for graphs with $n$ edges. We can further refine the space limit of the neuromorphic device to $\mathcal{O}(l(s,t))$ where $l(s,t) \leq n-1$ denotes the maximum path length between source and sink.
 
 In addition to the formal analyses, we experimentally demonstrate that off-loading the search for augmenting paths to a neuromorphic processor (in this case the Intel Loihi processor) could potentially yield energy efficiency gains, while not sacrificing runtime complexity.
 
 The remainder of this paper is organised as follows: In section \ref{prelim} we will give an introduction to neuromorphic computing in general, and neuromorphic complexity analysis and neural algorithm design in particular. This will provide a grounding for the reader to understand the subsequent sections.
 
 Following this introduction, we will formalize our machine and neural model in section \ref{model} and give a formal definition of the \NwFlow\ problem in section \ref{formal_def}. We will show that under this model, the \NwFlow\ problem is intractable to be realized in neuromorphic hardware alone, but benefits from an interaction between a traditional processor and a neuromorphic co-processor.
 
 In section \ref{algorithm} we will describe our algorithm and follow up with a complexity analysis of this algorithm in section \ref{complexity}.\\ In section \ref{methods} we will give a description of the methodology for our empirical validation of the theoretical analyses and provide the results from our experiments. In section \ref{discuss} we discuss these results and evaluate the proposed pipeline and in section \ref{concl} we will end with some concluding remarks.
 
 The goal of this project is threefold, (1) to use the tools and methods from computational complexity theory to come up with a systematic way of evaluating the feasibility of implementing computational problems in neuromorphic hardware, (2) to demonstrate that we can potentially expand the application space of neuromorphic hardware by suggesting an alternative model of computation, (3) inform neuromorphic hardware designers about potential architectural design limitations in expanding the application space to the class of problems under scrutiny in this project. By satisfying these goals, we hope to demonstrate the potential of a top-down analytical evaluation pipeline in demystifying the application space of neuromorphic hardware.

\section{Preliminaries}\label{prelim}

\subsection{Neuromorphic Hardware}
Neuromorphic computing has been one of the proposed novel architectures to replace the von Neumann architecture that has dominated computing for the last 70 years \cite{mead1990neuromorphic}. These systems consist of low power, intrinsically parallel architectures of simple spiking processing units.

In recent years numerous neuromorphic hardware architectures have emerged with different architectural design choices \cite{akopyan2015truenorth,davies2018loihi,schemmel2012live,khan2008spinnaker}. These distinctions include digital vs. mixed-signal approaches, the underlying neural model (simple neurons vs. more complex neuronal models), the scale of the systems in terms of the size of the networks that can be simulated, the number of features they offer,  and whether the operation speed is accelerated or real-time. These design decisions are mostly motivated by the type of applications the original designer had in mind. For example, the BrainScaleS system \cite{schemmel2012live} is an accelerated mixed-signal system that operates at 10.000x biological real-time. This architecture is appropriate for simulating realistic biological processes over multiple time-scales (from the millisecond scale to years).

The SpiNNaker system, on the other hand,  \cite{khan2008spinnaker} has a digital architecture and runs at biological real-time. Due to its scalable digital architecture, it can run very large neural simulations and its real-time operation opens up opportunities for robotics applications.

A different approach is to focus on flexibility. An example of such an architecture is the Loihi chip \cite{davies2018loihi}, a neuromorphic chip developed by Intel Labs. The Loihi chip has a digital architecture inspired by a simple computational model of neural information processing: a network of leaky integrate-and-fire (LIF) neurons. In addition to that, it offers a wide array of features that enable users to build more complex neuronal models.

In this project, we aim to complement these more bottom-up approaches by a strictly top-down approach, in which we analyse the resource demands of a computational problem in terms of energy, time and space. We abstract away from specific hardware architectures and use a model of networks of LIF neurons to describe our computations on neuromorphic hardware. We will elaborate on this in the remaining parts of this section.

\subsection{Neuromorphic Complexity Analysis}

In traditional computing,  computational complexity theory gives an indication of the resources needed to solve a computational problem in terms of their input size under the traditional Turing machine model \cite{turing1937computable}. Through computational complexity theory, we are able to define classes of problems that require at least a certain amount of time and space resources, including methods and tools to analytically assign a specific computational problem to a certain class. This not only gives us a fundamental understanding of what types of problems can and cannot be efficiently solved but also provides us with an analytical approach to determine whether a new computational problem is efficiently solvable or not.

In contrast to traditional computing, we do not have a strong understanding of the types of problems that can and cannot be solved with neuromorphic hardware. The methods and tools from computational complexity could potentially be very useful in understanding the application space of neuromorphic processors, but traditional complexity theory might not be the ideal way to analyse the resource constraints of neuromorphic systems. The resources analysed in computational complexity - time and space - are coarse and derived from an abstract model of computation. More significantly, it does not capture the resource arguably of most interest when moving towards neuromorphic solutions: energy expenditure.

Efforts are underway in designing a neuromorphic complexity theory that is more equipped to describe the resource demands of neuromorphic processors \cite{kwisthout2018neuromorphic,DK19}. In section \ref{model}, we will build upon this work and formalise an alternative machine model in which a traditional Turing machine communicates with a neuromorphic oracle. In sections \ref{algorithm} and \ref{complexity} we will demonstrate how the maximum network flow problem can be mapped on this model and how this machine model can capture energy expenditure.

\subsection{Neural Algorithm Design}
In spiking neural networks, the weights can be trained (e.g. through spike-time dependent plasticity) or programmed. In the latter case, two approaches currently exist. One approach is to design a network of stochastic spiking neurons in such a way that it corresponds to an instance of a particular optimization problem, e.g.the travelling salesperson problem (TSP). We can achieve this by constructing basic circuits such as winner-take-all circuits or logic circuits. These circuits constrain the spiking behaviour of the network in such a way that it creates an energy landscape (distribution of spike-based network states) that leads to a fast convergence to the optimal solution \cite{jonke2016solving}.\\
Another approach is to carefully handcraft the network on the neuron level to obtain desirable signals and computation in the entire networks, which has proven to be successful for basic computational operations \cite{severa2016spiking,verzi2017optimization}. Notably \cite{severa2016spiking} introduced a simple discrete spiking neural model that is able to capture many interesting computational primitives, such as delays, spike-timing and leakages. We will use this model to implement our neuromorphic oracle.  We will describe this model in detail in section \ref{model}.

There is currently no straightforward way to scale these approaches up to more complex compounded problems such as the maximum flow problem. In sections \ref{model}, \ref{algorithm} \& \ref{complexity} we expand on previous neural algorithm design work and demonstrate a novel way of tackling more complex problems such as the maximum flow problem.

\section{Model \& Problem Definition}\label{model}

\subsection{Machine Model}
We introduce a new machine model consisting of two components: (1) a Turing machine $M$ with a read-only input tape and a working memory, (2) a neuromorphic oracle $O$, a formal depiction of a computing device that receives a spiking neural network (SNN) definition from $M$, can simulate this SNN, and outputs information in the form of specific spiking events. $O$ can be either a transducer (computing a function and writing the result on the output tape) or a decider (deciding a decision problem and writing `0' or `1' on the output tape). When the computation halts after a single oracle call we refer to the model as a {\em pre-processing} model; if the Turing machine can consult the oracle multiple times and use the oracle's output in its subsequent computations, we refer to this model as an {\em interactive} model.

Let $L$ be a language of yes-instances of problem $\Pi$ and $i$ be a specific instance of $\Pi$. Then the Turing machine $M$ implements an algorithm $A_L(i)$ that decides whether $i \in L$. In addition to its normal behaviour, given input $i$, $M$ can construct an encoding of any spiking network $\mathcal{S}_{L,i}$ that is subsequently communicated to and processed by $O$. Both $A_L(i)$ and $\mathcal{S}_{L,i}$ work under constrained resources $R_T$ and $R_S$ relative to $|i|$, where $R_T$ is a two-tuple $(\mathrm{TIME}, \mathrm{SPACE})$ and $R_S$ is a three-tuple $(\mathrm{TIME}, \mathrm{SPACE}, \mathrm{ENERGY})$. This way we can prevent the construction of $\mathcal{S}_{L,i}$ from being trivial. In order to respect resource constraints $R_T$ where $\mathrm{SPACE} < \mathrm{TIME}$  we introduce a working memory tape from which $A_L$ can read and write. This working memory will have size $R_T[\mathrm{SPACE}]$. In figure \ref{fig:model} the reader can find an illustration of this model.

\begin{figure}[h]
    \centering
    \includegraphics[width=0.6\textwidth]{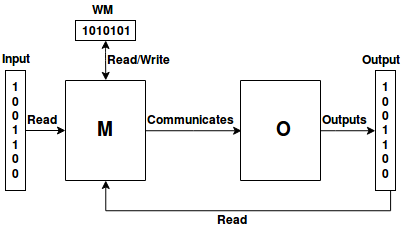}
    \caption{\textit{Illustration of the machine model. We have a conventional device $M$ which acts as a Turing machine with one input tape and a neuromorphic oracle $O$ to which $M$ can communicate a spiking neural network definition $\mathcal{S}_{L,i}$. If we have a specific language $L$ and input $i$ with imposed resource constraints $\{R_T, R_S\}$,  $M$ implements $A_L(i)$, can construct a spiking network $\mathcal{S}_{L,i}$, and can communicate this network to $O$. $O$ subsequently processes this network and outputs designated spiking events in the order of their spike timings (for a transducer model) or a single `0' or `1' (for a decider model). In order to constrain the space requirements of $M$ we only allow read access from the input tape of $M$ and we introduce a Working Memory (WM) with size $R_A[\textrm{SPACE}]$, where $R_A[\textrm{SPACE}]$ is the space constraint of $A_L$. In the pre-processing case the `Read' arrow from the oracle back to the Turing machine does not exist and the computation halts with the oracle writing the output on the tape.}}
    \label{fig:model}
\end{figure}

\subsection{Neural Model}
For the realisation of the oracle, we adopt the neural model in \cite{severa2016spiking}, in which a neuron $H_i$ is defined as a 3-tuple:
$$H_i = (T_i, R_i, m_i)$$
Where $T_i, R_i, m_i$ are the firing threshold, reset voltage and multiplicative leakage constant respectively.

A synapse is defined as a 4-tuple:
$$S_{a,b} = (d, w)$$
Where $a$ is the pre-synaptic neuron, $b$ is the post-synaptic neuron and $d$ and $w$ are the synaptic delay and synaptic weight respectively.

The spiking behaviour is determined by a discrete-time difference equation of the voltage. Suppose neuron $y$ has voltage $V_{ty}$ at time step $t$. Then we can compute the voltage at time step $t + 1$ in the following way:
$$V_{t+1y} = m_yV_{ty} + \sum_{S_{xy} exists}w_{xy}x_{t+1 - d_{xy}}$$
Where $x_{t+1 - d_{xy}} = 1$ if neuron $x$ spiked at time step $t + 1 - d_{xy}$ and $x_{t+1 - d_{xy}} = 0$ otherwise. A spike $x_t$ is abstracted here to be a singular discrete event, that is, $x_t = 1$ if a spike is released by neuron $x$ at time $t$ and $x_t = 0$ otherwise.

Additionally, we have  voltage $V_0$ that denotes the initial potential of a neuron. We assume $V_0 = 0$, unless explicitly mentioned otherwise. Furthermore, we assume that the membrane potential is non-negative.

We make a distinction between four types of neurons. A \textit{standard neuron} for internal computations a \textit{readout neuron}, from which the spike events will be written on the output tape of the neuromorphic oracle $O$ a scheduled neuron, a programatically defined neuron that fulfills a certain specific role (e.g.  scheduled firing or constantly firing) and an input neuron that represents the input in case the problem cannot be fully encoded in the neurons and synapses and needs external information to drive computation. In figure \ref{neurontypes} you can find an illustration of these neuron types. Note that in this project we only make use of the readout and standard neurons.

\begin{figure}[h]
    \centering
    \includegraphics[width=0.6\textwidth]{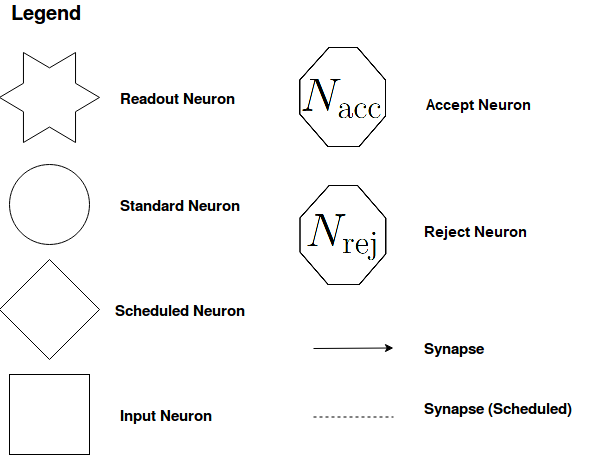}
    \caption{\textit{Illustration of the different neuron types. The standard neuron is used for internal computations in the network, while the readout neuron can submit their spiking events on the output tape of the oracle. The Accept and Reject neurons are specific to decider oracles and fire when the input accepts resp. rejects. The scheduled neuron and input neurons are auxiliary neurons that can be used to represent an external drive and to drive standard neurons with a bias current.}}
    \label{neurontypes}
\end{figure}

\subsection{Complexity classes}

In \cite{DK19} we introduced a hierarchy of complexity classes $\mathcal{S}(R_T,R_S)$ (for pre-processing neuromorphic oracles) and $\mathcal{M}(R_T')^{\mathcal{S}(R_T,R_S)}$ (for interactive neuromorphic oracles, where the `base' Turing machine is characterized by resources $R_T'$. In the context of this paper we are mostly interested in $R_T' = ( \mathcal{O}(n^c), \mathcal{O}(\log n))$; hence we refer to these classes as $\mathsf{L}^{\mathcal{S}(R_T,R_S)}$ using the familiar class of logspace problems $\mathsf{L}$.

\section{Network Flow on neuromorphic systems}
\label{formal_def}

Under the models defined above we define the following problem definitions for finding the maximum flow in a network:\\

\fproblem{\NwFlow\ (functional version)}
{A directed graph $G = (V,E)$, with designated vertices $s,t \in V$ referred to as the sink (no outgoing arcs) and source (no incoming arcs) of the network, respectively; for each edge $e \in E$ we have a non-negative integer $c(e)$, the capacity of that edge.}
{A flow assignment $f(e)$ for each edge $e \in E$ such that $0 \leq f(e) \leq c(e)$ and $\sum_e f(e)$ is maximised.}

\fproblem{\NwFlow\ (decision version)}
{As in the functional version; in addition; and integer $d$.}
{Is there a flow assignment $f(e)$ for each edge $e \in E$ such that $0 \leq f(e) \leq c(e)$ and $\sum_e f(e) < d$?}

We first illustrate that, using a naive approach, the decision version of \NwFlow\ is decidable on a pre-processing neuromorphic oracle, yet at the cost of exponential resources. In the subsequent sections we will show that we cannot decide this problem energy-tractably on such oracles, but provide a tractable algorithm for an interactive neuromorphic algorithm.

\subsection{A naive neuromorphic solution}

\begin{theorem}
Let $(G, d)$ be an arbitrary instance of \NwFlow\ with $n$ vertices and $m$ edges. Let $f_{max} = \max_{(u,v)} c(u,v)$ be the maximal flow possible between any two vertices in $G$; without loss of generality we may assume that $d \leq f_{max}$. $(G, d)$ is decidable in time $\mathcal{S}(\mathsf{DTIME}(\mathcal{O}(n\times(f_{max+1})^m))),(\mathcal{O}(f_{max}), \mathcal{O}(n\times(f_{max+1})^m), \mathcal{O}(n\times(f_{max+1})^m))$.
\end{theorem}

\begin{proof} 
We construct $\mathcal{S}$ from $(G, d)$ as follows. Let $\mathbf{f}$ be a joint assignment $f(u,v) \leq c(u,v)$ to every arc $(u,v)$. For every possible $\mathbf{f}$ such that $\sum_{u:(s,u)\in A} f(s,u) > d$ we construct a sub-network $\mathcal{S}_{\mathbf{f}}$; every network $\mathcal{S}_{\mathbf{f}}$ consists of $|V|$ sub-sub-networks $\mathcal{S}_{\mathbf{f},n}$ that tests whether $\sum_{u:(u,n)\in A} f(u,n) = \sum_{u:(n,u)\in A} f(n,u)$, i.e., that tests whether flow-in = flow-out for each node. The network structure for each $\mathcal{S}_{\mathbf{f},n}$ is given in Figure \ref{fin_fout}; basically, this network includes a neuron $f^n_{in}$ representing the flow-in of $n$, a neuron $f^n_{out}$ representing the flow-out of $n$, and a comparison circuit that realises that neuron $E_n$ fires at time ($\min$($\sum_{u:(u,n)\in A} f(u,n)$, $\sum_{u:(n,u)\in A} f(n,u)$)$ + 2)$ if and only if flow-in $\ne$ flow-out in node $n$. Each neuron $E_n$ is connected to a neuron $O_{\mathbf{f}} = (1,0,0)$ by a synapse $(0,1)$; $O_{\mathbf{f}}$ will fire at time at most $f_{max} + 3$ if any flow conservation constraint in $\mathbf{f}$ is violated. In $\mathcal{S}$, every neuron $O_{\mathbf{f}}$ is connected to $N_{\mathrm{rej}}$ by a synapse $(0,1)$; $N_{\mathrm{rej}}$ is wired such that it will keep firing once triggered (by a self-loop) and will keep inhibiting $N_{\mathrm{acc}}$, finally, $N_{\mathrm{acc}}$ will be scheduled to fire at time $f_{max} + 5$ unless inhibited by $N_{\mathrm{rej}}$. We conclude that $N_{\mathrm{acc}}$ fires if and only if $(G, d)$ is a yes-instance of \NwFlow, yet that the number of spikes, number of neurons, and the firing time are exponential in the size of $(G, d)$.
\end{proof}

\begin{figure}[h!]
\centering
\includegraphics[width=8cm]{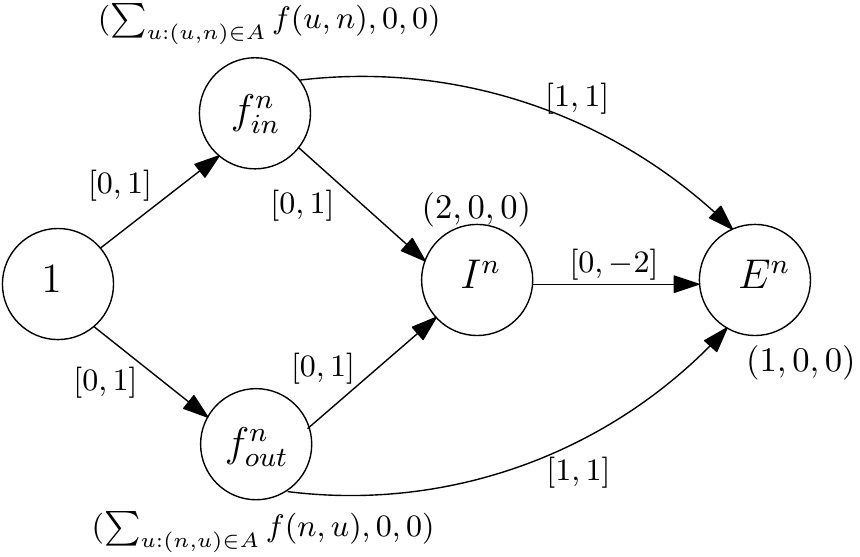}
\caption{The sub-sub-network $\mathcal{S}_{\mathbf{f},n}$ that tests whether $\sum_{u:(u,n)\in A} f(u,n) = \sum_{u:(n,u)\in A} f(n,u)$ for a given flow $\mathbf{f}$.}
\label{fin_fout}
\rule{\columnwidth}{0.3mm}
\vspace{1mm}
\end{figure}

\subsection{Intractability proof of a related problem}

In the remainder of this section we will give a lower bound for a more general variant of \NwFlow, namely the following problem:

\dproblem
{\TNwFlow}
{A directed graph $G = (V,A)$ with designated vertices $s, t \in V$, referred to as the sink (no incoming arcs) and source (no outgoing arcs) of the network, respectively; a capacity interval $c(u,v) = [c_{\mathrm{min}},c_{\mathrm{max}}]$ denoting respectively the minimum flow threshold $c_{\mathrm{min}}$ and capacity $c_{\mathrm{max}}$ of an arc $(u,v)$; dedicated auxiliary sinks $r \in R \subset V$ and sources $p \in P \subset V$ (together called reservoirs). Let the {\em flow} $f(u,v)$ be subject to 1) for all $(u,v) \in A$ either $f(u,v) = 0$ or $c_{\mathrm{min}}(u,v) \leq f(u,v) \leq c_{\mathrm{max}}(u,v)$ and 2) $\sum_{u:(u,v)\in A} f(u,v) = \sum_{u:(v,u)\in A} f(v,u)$ for all $v \in V \setminus (\{s, t\} \cup R \cup P)$. Let $d$ be a non-negative integer}
{Is $\sum_{u:(s,u)\in A} f(s,u) > d$}

Note that \NwFlow\ is a constrained version of this problem where for all arcs $c_{\mathrm{min}} = 0$ and $P, R = \varnothing$. We show that every spiking neural network $\mathcal{S}$ with $n$ neurons, time constraint $t$, and energy constraint $e \leq nt$ can be reduced using a linear reduction to an instance of \TNwFlow\ with $\mathcal{O}(nt)$ nodes. In particular, for {\em constant} time bounded network simulations (i.e., $t = \mathcal{O}(1)$), this implies that any $n$-node network simulation that runs in constant time (yet has no energy constraints) can be reduced to solving a \TNwFlow-instance with $n$ nodes. This implies that if \TNwFlow\ could be decided on a constant-time bounded neuromorphic device taking asymptotically {\em less} than maximum energy (e.g., $\sqrt{n}$ spikes per unit of time) then basically {\em every} constant-time bounded neuromorphic computation can be made more energy efficient (to $O(\sqrt{n})$ spikes per unit of time in this case) at the cost of only a linear amount of additional network size. In other words, this implies hardness of \TNwFlow\ for $\mathcal{S}((\mathcal{O}(n),\mathcal{O}(n)),( \mathcal{O}(1),\mathcal{O}(n),\mathcal{O}(n)))$. The full proof of this claim is elaborate and is delegated to the appendix.

In the next section, we show that we can solve \textsc{Max Network Flow} under this model with $R_A = (\mathcal{O}(n^3),\mathcal{O}(1))$ and $R_S = (\mathcal{O}(n),\mathcal{O}(n),\mathcal{O}(n))$, where $n$ is the number of edges in the flow network.

\section{Algorithm}\label{algorithm}
We adopt a variant of the Ford-Fulkerson Method: the Edmonds-Karp Algorithm \cite{edmonds1972theoretical}. This algorithm repeatedly finds shortest augmenting paths, pushes flow through these paths based on the edge with the minimum capacity until all augmenting paths have been exhausted. There are three components in this algorithm that violate the logspace constraint.
\begin{enumerate}
    \item The queue maintained by the search algorithm
    \item Maintenance of the path
    \item Maintenance of the flow through each edge
\end{enumerate}
In all three cases, the memory demands can be linear in the input in the worst case. We introduce one spiking operation and two neuromorphic data structures, through which we offload these components to the neuromorphic oracle. A modification of the wave propagation algorithm discussed in \cite{ponulak2013rapid} and two networks that maintain the augmenting path and the flow of the edges.  Algorithm \ref{EK} describes the algorithm in full. The algorithm is defined on the Turing machine $M$. $M$ can consult a neuromorphic oracle $O$ during execution.\\

\begin{algorithm}[h]
\KwData{weighted graph G = (V,E) with a capacity c for each edge}
    $Write\_Capacity(E)$\;\label{cap}
    \While{there is an augmenting path}{
        $Spike\_Search(E)$\;\label{search}
        $\var{Min\_Cap}$  = $\infty$\;
        $\var{Prev\_Neuron} = Null$\;
        \While {not $End(O)$} {\label{whilePathBegin}
            $\var{Neuron}$ = $Read(O)$\;
            \If {$Continuation(\var{Neuron}, \var{Prev\_Neuron)}$}{
                 \If{ $\var{Min\_Cap}$ $>$ $Cap(\var{Neuron})$}{
                    $\var{Min\_Cap}$ = $Cap(\var{Neuron})$\;}
                $Write\_Path(\var{Neuron})$\;
                $\var{Prev\_Neuron}$ = $\var{Neuron}$\;}}\label{whilePathEnd}
 
        \While{not $End(N)$}{\label{updateFlowBegin}
            $\var{Neuron}$ = $Read\_Path(O)$\;
            $Write\_Voltage(\var{Neuron} , \var{Min\_Cap})$\;
            $Write\_Capacity(\var{Neuron}$)\;}\label{updateFLowEnd}
    }
    $\var{Max\_Flow} = 0$\;\label{maxflowbegin}
    \While {not $End(O)$} {
        $\var{Neuron}$ = $Read(O)$\;
        $\var{Max\_Flow}$ += $Voltage(\var{Neuron})$\;
        }\label{maxflowend}
    return $\var{Max\_Flow}$ \label{return}
    \caption{Spiking E-K algorithm\label{EK}}
\end{algorithm}
Algorithm \ref{EK} gives a full description of a spiking version of the E-K algorithm. In line \ref{cap} we write away a readout capacity neuron for each edge on the neuromorphic oracle $O$, which keeps track of the flow that has gone through the neuron. The capacity neurons are defined as:
$$C_i  =  (c + (|E| + 1), 0, 1)$$
Where $C_i$\footnote{In this instance (and the remaining part of this section), we can set $V_0 = |E| + 1$ without loss of generality.} is the capacity neuron for edge $i$ in the flow network, $c$ is the capacity of the edge that the neuron codes for.  When $C_i$ reaches its firing threshold, it will fire exactly once and instantaneously inhibit its postsynaptic neurons.

In line \ref{search} we map our flow network onto two spiking networks in which each neuron codes for an edge in the flow network. We use the first network to find an augmenting path in the flow network and we use the second network to sort the edges in order, such that we can read out the path.  We define the neurons in the first network as:
$$H_i = (1 + (|E| + 1), 0, 1)$$
The reset of the neuron is set such that this neuron can only spike exactly once. The timing of the spike will then be proportional to the length of the shortest path from the source to the edge that the neuron codes for.

The connectivity of this network is defined as follows: let $a \rightarrow b$ and $c \rightarrow d$ be two edges in the original flow network with vertices $a, b, c$ and $d$. If $b = c$, we define a synaptic connection:
$$S_{H_{c\rightarrow d}, H_{a\rightarrow b}} = (1,1)$$
This means that the direction of flow in the spiking network is reversed w.r.t. the original flow network.  We then read out and connect the earlier defined capacity neurons according to:
$$S_{C_i,H_i}  = (0, -(|E| + 1))$$
In addition to that, we introduce a transmitter neuron $T$ that kick-starts the wave propagation:
$$T = (1, 0, 1)$$
And connect this transmitter neuron to the neurons that code for the sink edges according to:
$$S_{T, H_{*\rightarrow t}} = (1,1)$$
Where $t$ is the sink vertex.
The transmitter neuron will have $V_0 = 1$ at the start of the algorithm and will spike exactly once.
\\
\\
We define a second readout network with neurons:
$$R_i = (1 + (|E| + 1), 0, 1)$$
Where each neuron codes for a specific edge in the original flow network. The connection topology of this network is defined as follows: if $H_{s\rightarrow a}$ codes for a source edge in the first network we connect it to the second network according to
$$S_{H_{s\rightarrow a}, R_{s\rightarrow a}}  = (1,1)$$
In addition to that, we define the connectivity between neurons $R$ in the following way. let $a \rightarrow b$ and $c \rightarrow d$ be two edges in the original flow network with vertices $a, b, c$ and $d$. If $b = c$, we define a synaptic connection:
$$S_{R_{a\rightarrow b}, R_{c\rightarrow d}} = (1,1)$$
This means that the direction of flow from the flow network is preserved. This will enable us to read out the path in the correct direction.
We also connect the earlier defined capacity neurons according to:
$$S_{C_i,R_i}  = (0,-(|E| + 1))$$
This makes sure that the readout neuron will not spike if the capacity of the edge it codes for is exhausted.
Each spike event in this network will be written on the output tape of $O$. Note that in the second network, neurons will only fire if any of the neurons that code for the source edges in the first network fire. If this is not the case, it means that there is no path from source to sink left and we need to wait for $2\times|E| + 1$ time steps (the longest possible wave through both networks) in order to determine that we are done. The first half of the network thus guarantees that there is indeed a path from source to sink, and the second part of the network sorts edges in such a way that we can reliably decode the path from the network.
The neural algorithm will run until any of the readout neurons that code for a sink edge has spiked. If that is the case, it means that we found a path from the source to the sink. Otherwise, the algorithm will run for $2\times|E| + 1$ steps and terminate, which means that all the augmenting paths in the network have been exhausted.
\\
\\
In figure \ref{mapping} you can find an illustration of how a flow network is mapped on the described SNN topology.
\begin{figure}[H]
    \centering
    {\includegraphics[width = 2in]{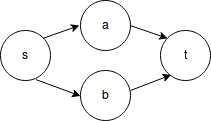}}
    {\includegraphics[width = 3in]{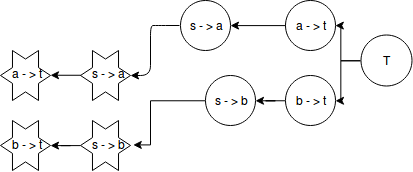}}
    \caption{\textit{Illustration of how a flow network is mapped on the SNN topology described in the text. The left graph depicts a simple flow network with a source $s$ and a sink $t$ and two intermediate nodes. On the right, you can see the SNN topology of the flow network. The first half of the SNN consists of standard neurons that compute the shortest path from the sink to the source. The second half consists of readout neurons that will reverse and sort the paths such that we can read out the path in the correct order. For clarity, we left out the capacity neurons that can inhibit the readout and standard neurons when the capacity of an edge is exhausted.}}
    \label{mapping}
\end{figure}
In lines \ref{whilePathBegin} - \ref{whilePathEnd} we read out the neurons and write them on the neuromorphic oracle as a path network. We check whether the neuron is a good continuation of the path, i.e. if we have two neurons $R_{a\rightarrow b}, R_{c\rightarrow d}$ we have a continuation if $b = c$. In addition to that, we keep track of the minimum capacity of the neuron we found in order to determine the minimum capacity edge of the path we found. We can trivially read out the path by letting the neurons spike and read out their spike events from the output tape of $O$.
\\
\\
In lines \ref{updateFlowBegin} -  \ref{updateFLowEnd} we read out the path data and update the capacity neurons based on the flow of the minimum capacity edge we found in lines \ref{whilePathBegin} - \ref{whilePathEnd}. We then write this neuron back as a capacity readout neuron on the neuromorphic oracle.
\\
\\
In lines \ref{maxflowbegin} - \ref{maxflowend} we read out the voltages from the capacity neurons in order to determine the final flow through each edge. We then sum them up to determine the maximum flow through the network and return that value in line \ref{return}.

\section{Computational Complexity Analysis}\label{complexity}
In this section, we discuss the complexity analysis of $A_L$ and $\mathcal{S}_{L,*}$. We strictly divide resources between both processors, meaning that if a computation is only counted towards the complexity of the processor if the computation happens on that processor. We will assume that the communication resources count towards the complexity of $A_L$. We also assume that a read or write operation takes  $\mathcal{O}(1)$ time and  $\mathcal{O}(1)$ space. And finally, we assume that a spike exerts $\mathcal{O}(1)$ energy.
\subsection{Complexity of $A_L$}
\subsubsection{$A_L$ runs in $\mathcal{O}(n^3)$ Time}
We show that $A_L$ has time complexity $\mathcal{O}(n^3)$, where $n$ denotes the number of edges in the network. In line \ref{cap}, we write capacity neurons on $O$. Given that we have $|E|$ edges, this part is linear in the number of edges.  In line \ref{search}, we create a spiking network that can determine the shortest augmenting path in the flow network. In order to create this network, we need to read the edges from the input tape of $M$, create a neuron that codes for this edge and communicate it to $O$. Next, we need to connect each neuron to their neighbouring edges and their respective capacity neurons. We can achieve that by reading out each neuron in the spiking network, read out the respective capacity neurons and read out the search neurons and communicate a connection to $O$ if the identifiers of the neurons match. For each neuron, you need to check at most $2\times|E|$ neurons and matching the identifiers only takes $\mathcal{O}(1)$ time, which means that this procedure takes $\mathcal{O}(n^2)$ time. In lines \ref{whilePathBegin} - \ref{whilePathEnd} we read out the neurons that spiked in the spiking network and build up a path. Every operation within the while statement takes $\mathcal{O}(1)$ time. Since there are $|E|$ readout neurons the complexity of the entire loop will be $\mathcal{O}(n)$ time. Similarly in lines \ref{updateFlowBegin} - \ref{updateFLowEnd} we only read out neurons that code for the path, which has size at most $|E|$, which means that we also only need  $\mathcal{O}(n)$ time. Computing the maximum flow in lines \ref{maxflowbegin} - \ref{maxflowend} then also only takes $\mathcal{O}(n)$ time.
The outer loop depends on the number of augmenting paths. Since in each iteration, one edge will be saturated, there are $\mathcal{O}(|E|)$ possible paths, which is still polynomial in the input. Given that dominating time complexity within the loop is $\mathcal{O}(n^2)$, the entire procedure runs in $\mathcal{O}(n^3)$ time.
\subsubsection{$A_L$ runs in $\mathcal{O}(1)$ Space}
We show that the algorithm described in section \ref{algorithm} only uses $\mathcal{O}(1)$ space for the preprocessor, besides the encoding of the input. In section \ref{algorithm}, we identified three bottlenecks of E-K algorithm. We will address how all three bottlenecks are resolved. Since the search algorithm is entirely written off to the neuromorphic oracle, there is no in-memory maintenance of any sort of queue so we only use $\mathcal{O}(1)$ memory. Likewise, we do not fully maintain the path in memory but use the property that the neurons are sorted according to their spike timing. We then read every neuron and determine on the basis of their ID's whether it is a correct continuation and we can at the same time record the capacity encoded by the neuron in order to determine the minimum flow. Hence we only need to allocate $\mathcal{O}(1)$space. Likewise, since we code the flow into the threshold of the capacity neurons, we do not need to maintain it in the WM of $M$, since we can read and communicate flow information from and to the oracle $O$, which means that we only need to allocate $\mathcal{O}(1)$ space. This means that we can run this algorithm using only $\mathcal{O}(1)$ memory on the preprocessor and therefore $A_L$ runs in $\mathcal{O}(1)$ space.
\subsection{Complexity of $\mathcal{S}_{L,*}$}
\subsubsection{$\mathcal{S}_{L,*}$ runs in $\mathcal{O}(n)$ Time}
The time complexity of $\mathcal{S}_{L,*}$ is dominated by the search procedure since retrieving information from the path and capacity neurons can be trivially solved by letting them spike in $\mathcal{O}(1)$ time. The worst case that can occur consists of one path (i.e. a chain of nodes). In that case, we need $2\times|E| + 1$ time steps to receive an output from the oracle, which reduces to $\mathcal{O}(n)$ time. Since we have that $|E| = |V| - 1$ when the network is a chain of nodes, we can further refine that to $\mathcal{O}(l(s,t))$  where $l(s,t)$ denotes the path length from source to sink.  
\subsubsection{$\mathcal{S}_{L,*}$ runs in $\mathcal{O}(n)$ Space}
We need to allocate $|E|$ edges for the capacity neurons. For the path neurons, we need to allocate at worst $|E|$ neurons and for the search network, we need to allocate $2\times|E|$ neurons. Which results in $4\times|E|$ neurons in total. Which means that we need $\mathcal{O}(n)$ space.
\subsubsection{$\mathcal{S}_{L,*}$ uses $\mathcal{O}(n)$ Energy}
Each neuron in the search network spikes at most one time. Since we have $3\times|E| + 1$  neurons (including the capacity neurons) in the networks we need $\mathcal{O}(n)$ energy.
\\
\\
From the above description we have that  $R_A = (\mathcal{O}(n^3),\mathcal{O}(1))$ and $R_S = (\mathcal{O}(n),\mathcal{O}(n),\mathcal{O}(n))$  It therefore must be the case that \NwFlow\ is in $\mathsf{L}^{\mathcal{S}(\mathsf{L},(\mathcal{O}(n),\mathcal{O}(n),\mathcal{O}(n)))}$. Since the time complexity of the search query depends on the longest path between source and sink we can further refine this result to $\mathsf{L}^{\mathcal{S}(\mathsf{L},(\mathcal{O}(l(s,t)),\mathcal{O}(n),\mathcal{O}(n)))}$ where $l(s,t) \leq n-1$ denotes the maximum path length between source and sink.
\section{Empirical Analysis}\label{methods}
In this section, we will describe our empirical methodology. The goal of our empirical investigation is to validate our theoretical complexity results in actual neuromorphic hardware. In this way, we will be able to analyse whether our formal complexity results translate to practical reality and if our proposed theoretical machine model can be mapped onto neuromorphic hardware. This analysis consists of two phases. In the first phase, we will validate our complexity results in a neural simulator we have developed. This phase will serve as a sanity check prior to the hardware implementation and is primarily aimed at testing the algorithm we defined in section \ref{algorithm}. In the second phase, we will implement our algorithm on the Loihi chip. Our aim in this phase is to determine whether (1) our energy analysis holds on actual neuromorphic hardware, (2) the communication line between conventional processor and neuromorphic coprocessor introduces significant additional overhead that is not captured in our machine model, (3) it is practical to implement our algorithm under the proposed machine model and to map out the compromises we have to make in order to end up with a workable solution.
\subsection{Software Validation}\label{softwareval}
We will first describe our methodology w.r.t software validation. We implemented a variant of the proposed algorithm in a neural simulator we have developed for the purposes of this project. We look at the part of the algorithm that dominates resources constraints. This will be the search algorithm. We will compare the BFS algorithm on a conventional processor against the search algorithm implemented under our machine model in terms of time and energy resources. For our machine model, we will split this into two parts, the part of the algorithm that runs on the conventional processor and the part of the algorithm that runs on the neuromorphic co-processor. On the neuromorphic co-processor, we measure the time, space and energy demands of the spiking network.
\\
We compute two measures of interest after running each spiking network. The first measure is the number of total spikes in the network, which serves as a proxy for the energy demands.The second measure is the number of time steps needed until the sink neuron spikes in order to see how the time demands of the spiking networks increase with the number of nodes.  We also look at the absolute difference of the flow computed by the spiking network and a reference maximum flow computed with the Edmonds-Karp algorithm \cite{edmonds1972theoretical} in order to validate the accuracy of the solution. For the part of the algorithm that runs on the conventional processor, we need to take into account the resources that it takes to write the network onto the tape. As the size of the network does not change over execution time, we can abstract away and estimate the resources in terms of time.\\
In order to estimate energy demands on the conventional processor, we will assume that the energy demands are bounded by the time complexity on the conventional processor by a constant. This approximation is based under a different (but equivalent) machine model in which each operation can be decomposed into a set of primitive manipulations on the register \cite{cook1973time}. We will use this same method for the search algorithm that is fully implemented in the conventional processor.\\
We will compare two types of graphs. Graphs that have a low degree of connectivity relative to their number of nodes (i.e. sparse graphs) and graphs that have high connectivity relative to their number of nodes. This allows us to untangle in which instances the hybrid model might outperform the conventional method.\\
Since we are interested in how these measures grow when the size of the network increases we compute these measures over a series of networks with an increasing number of nodes. For each size, we randomly generate flow networks according to the following procedure\footnote{All non-proprietary software developed for this project can be found \href{https://github.com/AbdullahiAli/SpikingFlowNet}{\textbf{here}}.}.
\begin{enumerate}
    \item Specify the number of nodes $n$ and the number of edges $m$ in the network
    \item generate $n$ nodes, for each node identifier $i$ (a natural number): connect $i$ to all nodes
    with identifier $i+1$ or higher.
    \item The above procedure will yield a fully connected acyclic component. If this component contains more than $m$ edges: pick a random edge and delete it from the network. If the remaining network is not a fully connected component, put the removed edge back in the network. Repeat this procedure until you end up with a network of $m$ edges.
    \item Specify a maximum capacity $c_{max}$ and randomly assign a capacity $[1, c_{max}]$ to each edge.
\end{enumerate}
We will then run our algorithm on the generated networks, compute the aforementioned measures and average them out over these networks in order to obtain an average estimate for each network size. The maximum capacity and connection density is fixed between network sizes in order to obtain comparative results.
\subsection{Software Results}\label{softwareres}
There was no divergence between the conventional E-K algorithm implementation and the spiking version of the E-K algorithm, indicating that the spiking version works correctly. In figure \ref{softwarebench} you find the results of the software simulation.

\begin{figure}[H]
    \centering
    {\includegraphics[width = 4in]{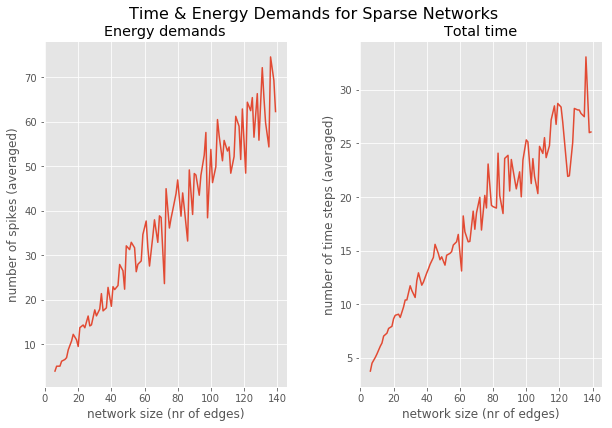}}
    {\includegraphics[width = 4in]{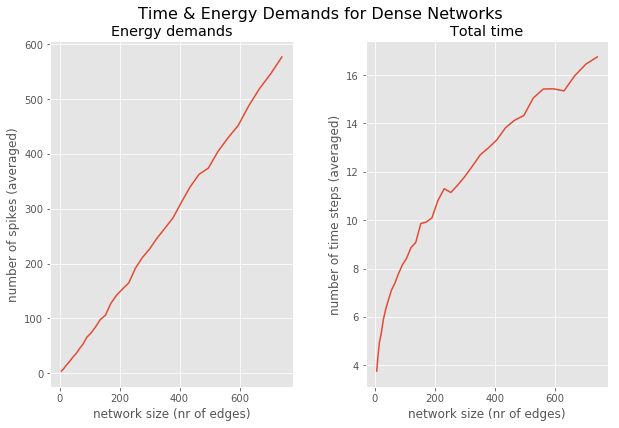}}
    \caption{\textit{Average growth of energy expenditure and time complexity of randomised flow graphs as a function of the number of edges under the hybrid model. We tested two types of graphs. Sparse graphs, i.e. $|E| = 1.4|V|$ and dense graphs, i.e. $|E| = \frac{|V|(|V| - 1)}{2}$. For sparse graphs, we generated graphs with 5 - 100 nodes and for the dense graphs, we generated graphs with 5-40. Each point in each graph is an average over the mean number of time steps and spikes over the search queries. }}
    \label{softwarebench}
\end{figure}
Since there is no obvious first-order proxy for energy consumption on a conventional model at the algorithmic level, we will first lay out our assumptions in this comparison. We assume that resources spent in moving data, i.e. from conventional processor to neuromorphic coprocessor in the hybrid model, and from RAM to CPU in the conventional model will be roughly proportional to each other and this will be left out of this analysis. We will also assume that the energy complexity of the conventional model will be proportional to the time that it takes to execute a certain computation. That is, if a certain computation takes $k$ time steps, the energy expended will be counted as $ck$ where $c$ is certain nonnegative constant. As already mentioned earlier we will use the BFS algorithm \cite{moore1959shortest} as a benchmark which has worst-case time complexity
$$R_{BFS}[TIME] = \mathcal{O}(|V| + |E|)$$

Under our aforementioned assumptions this will lead to the following energy expenditure:
$$R_{BFS}[ENERGY] = c|V| + c|E|$$
In figure \ref{softwarebench} we observe that for sparse networks, energy grows as $\frac{|E|}{2}$. If we assume $c > 0$ this means that for sparse networks we might predict an improvement in energy efficiency since $\frac{|E|}{2} < c|E| + c|V|$, but hardware validation needs to confirm this observation. In dense networks, we see a strict linear growth in terms of energy. Since we have that $|E| = \frac{|V|(|V| - 1)}{2}$. We end up with the inequality:
$$c|V| + c\frac{|V|(|V| - 1)}{2} > \frac{|V|(|V| - 1)}{2}$$
Which will be satisfied if $c > 0$, which means that if the energy expenditure is a multiple of the number of time steps, we can predict that we would also see improvements in terms of energy efficiency for dense nets. \\ \\
In terms of time complexity we can see in figure \ref{softwarebench} that both networks type show growth of $log(|E|)$, which is strictly more efficient than the conventional model. This indicates strong evidence for an improvement in terms of time complexity in the hybrid model.

\subsection{Hardware Validation}
In section \ref{softwareval} we concluded that there are potential efficiency gains in time and energy in the hybrid model. In order to arrive to this conclusion we made two assumptions: (1) the communication overhead between the conventional processor and neuromorphic processor is negligible, (2) one operation on the neuromorphic processor is proportional to one unit of energy and therefore one operation on the neuromorphic processor will be more efficient than an operation on the conventional processor. In this section, we will attempt to verify these two assumptions. \\
\\
We implemented the described algorithm in section \ref{algorithm}  on the Loihi Nahuku board. In order to meet the restrictions of the API of the Loihi processor we only implemented the search algorithm on the neuromorphic cores. Since the search algorithm is the major resource consumer in this algorithm we do not expect this revision to have major effects on the obtained results. We benchmarked the Nahuku board on two different levels of detail, the macro -and micro level. On the macro-level, we look at the entire system (neuromorphic chip and auxiliary systems) and measured network set-up, compilation and execution time. The set-up and compilation time gives us an estimate of how much resources the communication line between the conventional -and neuromorphic processor consumes and the runtime will give us an estimate of the execution time on the Loihi chip. On the micro-level, we zoom in and look at energy and execution performance on the neuromorphic chip. This will give us a more detailed estimate of how the energy and runtime demands evolve for larger networks. We repeated the experiments in section \ref{softwareres}  under the same conditions.

\subsection{Macro-level Results}
In figure \ref{macro} you can find the results for the macro-level benchmarks. We tested sparse and dense connected network as formulated in section \ref{softwareres}, each data point is an average of 10 randomly sampled flow networks.\\ For sparse networks, we see an initial high offset in execution time followed by slow growth and a negligible cost in communication overhead (setup time and compile time).\\
For dense networks, however, we see a higher cost in runtime. This can be explained by the fact that the number of edges grows much faster in dense networks compared to sparse networks. Note that in our original algorithm we could stop the execution on the basis of a signal of the neuron (spike of a source neuron). This is not possible in the Loihi API so we had to upper bound our execution time to the worst possible outcome (each neuron spikes before we have a solution). This, in particular, deteriorates the runtime results for dense networks. For dense networks, we again see a negligible cost in communication overhead.
\begin{figure}[H]
    \centering
    {\includegraphics[width = 3in]{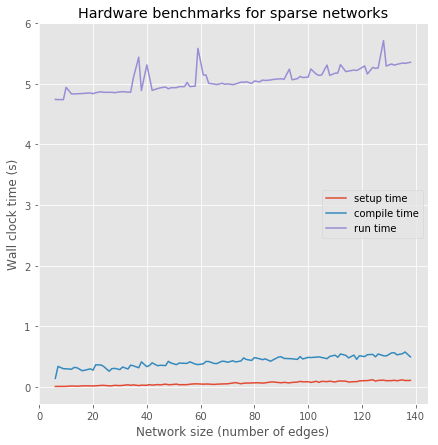}}
    {\includegraphics[width = 3in]{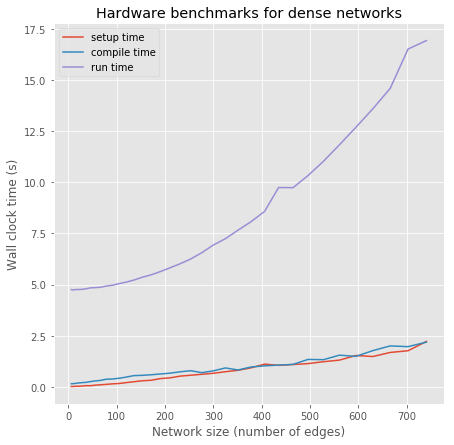}}
    \caption{\textit{Macro-level benchmark of Loihi system. We benchmarked the set-up time of the network (mapping flow net to a NxNet definition), compilation-time of the network and the execution time. On the x-axis, we have the number of edges in the original flow network and on the y-axis the wall clock time in seconds. For both network types, we see a constant/fairly slow growth in runtime indicating that the search algorithm scales very well in terms of actual simulation on the neuromorphic chip. In terms of set-up and compilation we see that both network types show slow scale-up as predicted by our previous simulations.}}
    \label{macro}
\end{figure}

\begin{figure}[H]
    \centering
    \includegraphics[width=0.8\textwidth]{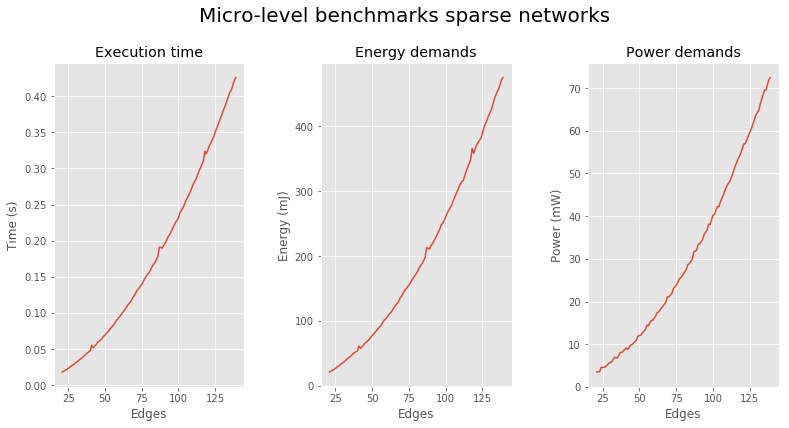}
    \caption{\textit{Micro-level benchmarks on the Loihi chip. Sparse networks varying from 15 nodes up to 100 nodes were benchmarked. We measured the execution time, energy demands and power demands. Each data point consists of 10 randomly sampled flow networks. On the x-axis are the number of edges in the network and on the y-axis are the magnitudes of interest. Execution time is measured in seconds, energy in milli-Joules and power in milliwatts.} }
    \label{micro_sparse}
\end{figure}

These results demonstrate that communication overhead is negligible, verifying our first assumption. In the case of sparse networks, we even see an improvement in runtime results.

\subsection{Micro-level Resuts}
We benchmarked sparse networks varying from 15 nodes up to 100 nodes and dense networks varying from 15 to 40 nodes. We measured the execution time, energy demands and power demands. Each data point consists of 10 randomly sampled flow networks and in figure \ref{micro_sparse} and \ref{micro_dense} you can see the results obtained from these benchmarks.\\
There is a slightly super-linear relationship between the network sizes and the measured statistics, but note that the execution times and power and energy scale in exactly the same way. These results fall in line with our second assumption and show that running the search query on a neuromorphic chip will be more energy-efficient, under the assumption that an operation on a neuromorphic chip will consume less energy.

\begin{figure}[H]
    \centering
    \includegraphics[width=0.8\textwidth]{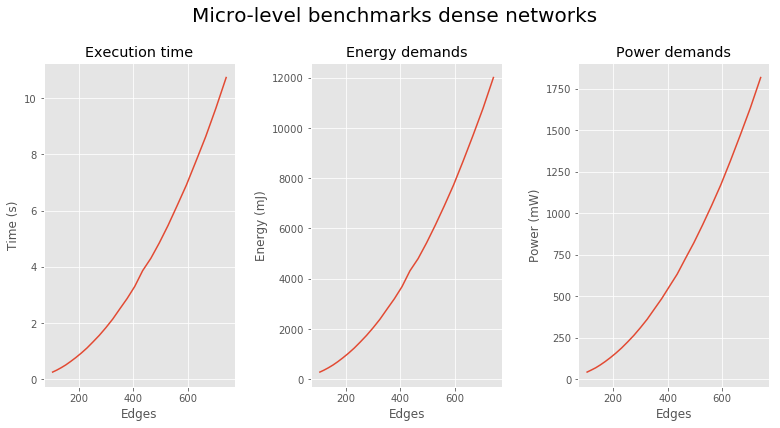}
    \caption{\textit{Micro-level benchmarks on the Loihi chip. Dense networks varying from 15 nodes up to 40 nodes were benchmarked. We measured the execution time, energy demands and power demands. Each data point consists of 10 randomly sampled flow networks. On the x-axis are the number of edges in the network and on the y-axis are the magnitudes of interest. Execution time is measured in seconds, energy in milli-Joules and power in milliwatts.}}
    \label{micro_dense}
\end{figure}

\section{Discussion}\label{discuss}
We have demonstrated a potential pipeline from theory to practical implementation in order to systematically investigate potential application areas for neuromorphic processors. Below we will discuss the advantages and limitations of this pipeline.

\subsection{What can we learn from theory?}
From classical complexity \cite{goldschlager1982maximum}, we learned that the \NwFlow\ problem is hard to implement on contemporary neuromorphic processors due to an inherently serial component in the problem. In sections \ref{model} to \ref{complexity} we demonstrated a theoretical approach in which we introduced a new machine model, with a new lattice of complexity classes, in which we unfold a new source of complexity: energy. On the basis of this model, we proposed an algorithm that shows that we can satisfy the logspace constraint. Something that was not possible under the classical Turing model.\\
In addition to that, under the assumption that one operation takes one unit of energy and communication overhead is negligible, we unveiled that we can get potential efficiency gains through off-loading parts of the algorithm on a neuromorphic co-processor.\\
While some of the theoretical results did not hold when moving to practice, this at least shows in what way such a theoretical approach could be beneficial. It helps us understand what aspects of a problem make a problem hard or easy to implement in neuromorphic hardware. Moreover, it gives us the ability to generalize this understanding to a larger class of problems (in this case the class of $\mathsf{P}$-complete problems). And finally, it allows us to come up with an alternative way to solve this problem and gives us pointers to potential efficiency gains we can get in terms of time, space and energy. Especially the last point is hard to arrive at without a fundamental understanding of the hardness of a certain problem. Theory can, therefore, provide a sound and rigorous basis on
which we can motivate why a certain application area is suited for neuromorphic solutions and help us come up with novel ways to solve a certain problem based on our understanding of the sources of complexity in the problem.

\subsection{From theory to practice: what do we sacrifice and what do we learn?}
In section \ref{methods} we continued to build on our theoretical results and implemented the algorithm we described in section \ref{algorithm}. We split this part up in two phases: a software phase and a hardware phase. The software phase served as an initial sanity check. It helped us spot potential problems in the proposed algorithm and understand in what way we have to modify our algorithm to make it work on neuromorphic hardware. We notably only offloaded the search part on the neuromorphic processor and did not include the maintenance of the flow and the maintenance of the path. This violates the logspace constraint, but through the results in section \ref{softwarebench} we were able to predict potential efficiency gains in time and energy. \\ \\
An important comment to make is that these predictions only held under two assumptions: (1) negligible communication overhead, (2) a linear relationship between runtime and energy. In the hardware phase, we verified these assumptions and ran a modified algorithm on the Loihi platform. We split these analyses up in a macro and micro part. In the macro part, we showed that in the full system (i.e. Loihi chip + network setup, compilation and sending and handling jobs) the biggest source of complexity is the runtime. In both networks, we found that the communication overhead was negligible. This indicates that our first assumption holds. In the case of sparse networks, we even found that the runtime grows very slow w.r.t. the network size indicating potential runtime efficiency gains. Important to note is that the overall runtime has an initial high cost but subsequently grows rather slow indicating that a hybrid approach is only cost-efficient if the networks are somewhat of large scale. \\
In the micro part, we zoomed in on the Loihi chip and looked at how the execution time, energy demands and power demands grow as the network sizes grow. In figures \ref{micro_sparse} and \ref{micro_dense} we see that there is a roughly linear relationship with time, energy and power, this means that our earlier prediction that time would scale roughly logarithmically with network size does not hold (see figure \ref{softwarebench}, but it does confirm that energy demands roughly scale linearly with network size. This means that our assumption that one operation is roughly proportional to one unit of energy holds. Since the Loihi processor is much more energy efficient than conventional processors \cite{davies2018loihi}, we have strong evidence that off-loading the search query to a neuromorphic processor yields efficiency gains while not sacrificing runtime complexity for relatively large scale networks.
\\
\\
The above discussion shows that software validation phase is a good complement to actual hardware benchmarking. It helps you understand limitations in your algorithm and it helps you pinpoint under what circumstances you might see efficiency gains. This in term helps in interpreting subsequent results obtained in hardware. We, therefore, see value in explicitly incorporating a software validation phase in the pipeline.

\subsection{Future work: A more comprehensive complexity theory and new computational problems}
From our results, several pointers of future research arise. We need a more comprehensive neuromorphic complexity theory. That includes hardness proofs, a notion of completeness, complexity classes and a means to reduce problems to other problems while preserving essential properties of the problem (e.g. time and energy) and potentially models that unfold different type of sources of complexity. Work is already done in this direction (e.g. see \cite{kwisthout2018neuromorphic}), and this work would enhance our fundamental understanding of what makes a problem efficiently solvable on a neuromorphic processor and would greatly help us in mapping the space of potential applications for neuromorphic hardware.\\
In addition to that, we need to investigate new computational problems, that can lead to new algorithm design patterns such as the hybrid approach we proposed for the maximum flow problem. This, in turn, could enhance the programming tools available to neural algorithm designers.

\section{Conclusion}\label{concl}
In this project, we described a pipeline from complexity theory to practical implementation in order to systematically explore the application space of neuromorphic processors. We picked the maximum flow problem \cite{ford1955simple}, a more complex algorithm than previous algorithms studied in the neural algorithm design field \cite{verzi2017optimization,amione2018nonneural}.\\
By introducing a hybrid computational model, we were able to show that \textsc{Max Network Flow} is in $\mathsf{L}^{\mathcal{S}((\mathcal{O}(n^c),\mathcal{O}(\log n)),(\mathcal{O}(n),\mathcal{O}(n),\mathcal{O}(n)))}$. This means that we have found a link between $\mathsf{P}$-complete problems and a machine model in which we potentially could reduce the space requirements of these problems.\\
Through our practical analyses, we were able to confirm that there are potential efficiency gains by off-loading the search procedure to a neuromorphic processor, while not sacrificing runtime complexity. Additionally, the practical investigation also showed what comprises were needed in order to implement this algorithm on neuromorphic hardware. Most notably, the logspace constraint was violated.\\
This shows that theory and practice should ideally be tightly interlinked. Theoretical analyses help us understand why certain problems can or cannot be efficiently implemented in neuromorphic hardware, and can help us in coming up with novel ways of solving problems. Practical investigations then help us refine our algorithm and/or theoretical model. Ideally when employing this pipeline, one should iterate back and forth from theory to practice.\\
Future endeavours would include, a more comprehensive neuromorphic complexity theory that would better allow us to map out the application space of neuromorphic hardware systems and new neural algorithm design patterns that could help us tackle problems in novel ways.

\section*{Acknowledgements}

This work is partially based on the MSc in Artificial Intelligence thesis of the first author. We are grateful to Nils Donselaar and Iris van Rooij for valuable comments and feedback on earlier versions of this paper. This project was supported by financial and practical support from Intel Corporation via the Intel Neuromorphic Research Community.

\appendix
\section{Appendix}

We claimed that every spiking neural network $\mathcal{S}$ with $n$ neurons, time constraint $t$, and energy constraint $e \leq nt$ can be reduced using a linear reduction to an instance of \TNwFlow\ with $\mathcal{O}(nt)$ nodes. We hereby make the following assumptions about constraints on the behavior of $\mathcal{S}$ and adaptations to the LIF model introduced above for technical reasons in the proof below:

\begin{enumerate}
\item $\mathcal{S}$ is constructable with a Turing machine in time, independent of the input size (i.e., in constant time);
\item $\mathcal{S}$ contains exactly one constant (always firing) input $C$ and no neurons have additional biases;
\item All weights, delays, and thresholds are non-negative integers;
\item All leakages constants are set to $1$;
\item Instead of setting neuron potentials to their reset value when a neuron fires, the potential will be set to the `overflow' after firing, i.e., the threshold value is subtracted from the potential and the reset value is the remainder;
\item $N_{\mathrm{rej}}$ fires until $N_{\mathrm{acc}}$ fires and then remains silent.
\end{enumerate}

In our construction we use several `$C$-gadgets' that ensure that a specific amount $f_i$ of flow can pass through a source-sink pair $<s_i, t_i>$ if and only if a specific global constraint $C$ on the behaviour of $\mathcal{S}$ is met. In addition we make `local' constructs that mimic the dynamics of potentials and spike transmission inside and in-between neurons. Finally we combine these parts and show that a flow $f$ can go from the `master-source' to the `master-sink' if and only if $\mathcal{S}$ accepts, i.e., $N_{\mathrm{acc}}$ fires before time $t$ taking energy at most $e$. We build up our construction of the flow network $G$ as follows. We designate neurons $N_{\mathrm{acc}}$, $N_{\mathrm{rej}}$, and $N_{\mathrm{con}}$ as acceptance, rejection, and constant neuron, respectively.

\begin{enumerate}
\item[1.] For the constant input $N_{\mathrm{con}} \in N$ we include $t$ vertices $n_{\mathrm{con},k}$ in $G$. For the acceptance and rejection neuron we similarly include $t$ vertices $n_{\mathrm{acc},k}$ and $t$ vertices $n_{\mathrm{rej},k}$. We introduce one auxilliary source $p_{\mathrm{con}}$ and a vertex $n^{\mathrm{src}}_{\mathrm{con}}$, connect $p_{\mathrm{con}}$ to $n^{\mathrm{src}}_{\mathrm{con}}$ with capacity $[t, t]$ and connect $n^{\mathrm{src}}_{\mathrm{con}}$ to all $t$ vertices $n_{\mathrm{con},k}$ with capacity $[1,1]$. This construct enforces that either all constant inputs fire, or none of them fires. 
\item[2.]For every neuron $N_i \in N \setminus \{N_{\mathrm{con}}, N_{\mathrm{acc}}, N_{\mathrm{rej}}\}$, we include $t$ vertices $n_{i,k}$ in $G$. We introduce an arc between every vertex $n_{j,k}$ and $n_{j,k+1}$ for $1 \leq k < t$, with capacity $c(n_{j,k},n_{j,k+1}) = [0, T_{n_j} - 1]$. 
\item[3.] For every {\em outgoing} synapse $s_{a,b} = (d,w) \in S$ and for all time points $1 \leq k + d < t$ we introduce an synapse-gadget (later to be defined) $s_{a,k}$ between $n_{a,k}$ and the receiving nodes $n_{b,k+d}$.
\end{enumerate}

Figure \ref{general_structure} shows the thus constructed general structure of the flow graph after these steps.

\begin{figure}[h!]
\centering
\includegraphics[width=9.8cm]{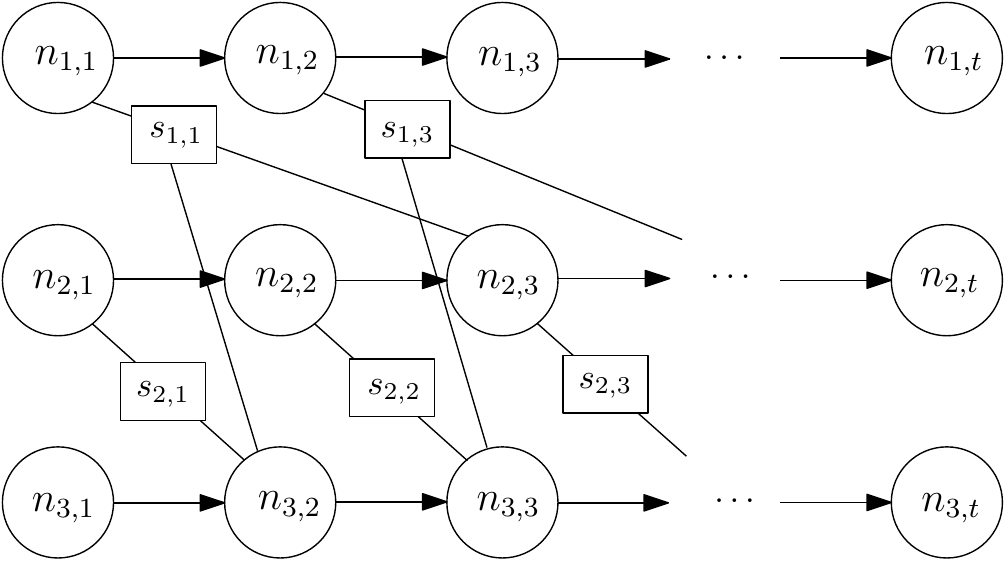}
\caption{The general structure of $G$ after the first two steps exemplified for the neurons $n_1, n_2, n_3$ and the synapses $s_{1,2} = (2, w_{1,2})$, $s_{1,3} = (1, w_{1,2})$ and $s_{2,3} = (1, w_{2,3})$.}
\label{general_structure}
\rule{\columnwidth}{0.3mm}
\vspace{1mm}
\end{figure}

\begin{enumerate}
\item[4.] We introduce an energy gadget $E$ as follows. We introduce vertices $j_m, 1 \leq m \leq j$, with a source $s_e$ and a sink $t_e$, and define the capacity between $s_e$ and $j_1$ and between $j_e$ and $t_e$ to be $[0,1]$, and between $j_m$ and $j_{m+1} (m < e)$ to be $[0,e]$. Furthermore, we introduce an auxilliary sink $r_e$ and connect $j_e$ to $r_e$ with capacity $[0,t-1]$. There will be an arc from every synapse-gadget $s_{a,b,k}$ to $j_k$ with capacity $[0,1]$ (figure \ref{energy_gadget}). The intuition here is that we can push flow $1$ from $s_e$ to $t_e$ if and only if at most $e-1$ neurons fire over the course of the simulation.
\end{enumerate}

\begin{figure}[h!]
\centering
\includegraphics[width=10cm]{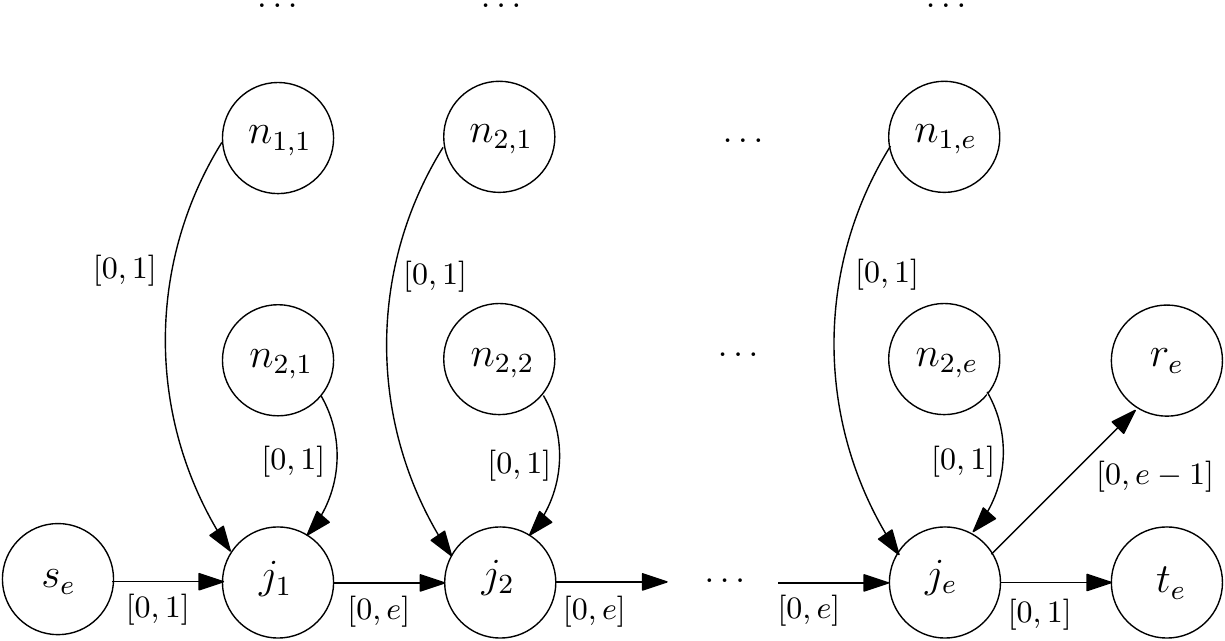}
\caption{Energy gadget $E$ ensuring that a flow of $1$ can be pushed from $s_e$ to $t_e$ if and only if the energy constraint is satisfied.}
\label{energy_gadget}
\rule{\columnwidth}{0.3mm}
\vspace{1mm}
\end{figure}

\begin{enumerate}
\item[5.] We introduce an time gadget $T$ as follows. We introduce vertices $h_l, 1 \leq l \leq t$, with a source $s_t$ and a sink $t_t$, and define the capacity between $s_t$ and $h_1$ and between $h_t$ and $t_t$ to be $[0,1]$, and between $h_l$ and $h_{l+1} (l < t)$ to be $[0,t]$. Furthermore, we introduce an auxilliary sink $r_t$ and connect $h_t$ to $r_t$ with capacity $[0,t-1]$. There will be an arc from the vertex representing the rejection state at any point in time $N_{\mathrm{rej},k}$ to $t_k$ with capacity $[0,1]$ (figure \ref{time_gadget}). The intuition here is that we can push flow $1$ from $s_t$ to $t_t$ if and only if $N_{\mathrm{rej}}$ stops firing after at most $t-1$ time steps in the simulation.
\end{enumerate}

\begin{figure}[h!]
\centering
\includegraphics[width=10cm]{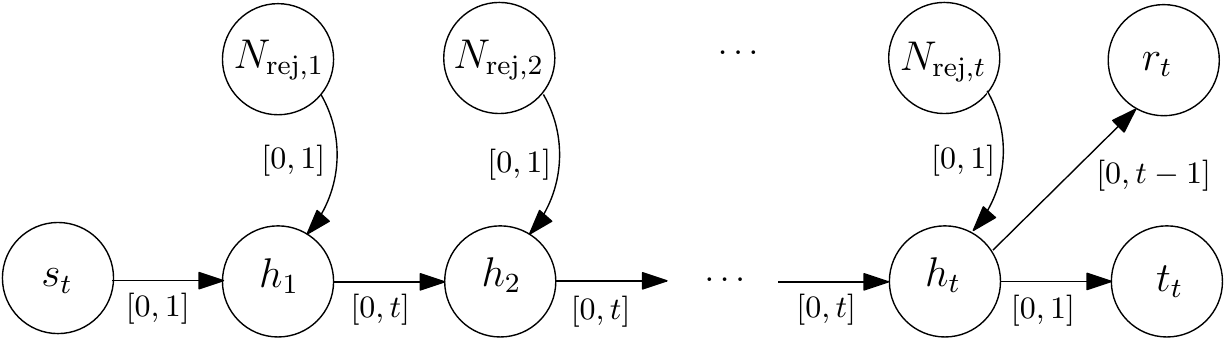}
\caption{Time gadget $T$ ensuring that a flow of $1$ can be pushed from $s_t$ to $t_t$ if and only if the time constraint is satisfied.}
\label{time_gadget}
\rule{\columnwidth}{0.3mm}
\vspace{1mm}
\end{figure}

\begin{enumerate}
\item[6.] We define the previously introduced synapse-gadgets $G_a$ as follows. We connect $n_{a,k}$ to $n^{\mathrm{out}}_{a,k}$ with capacity $[T_a,T_a]$. This ensures that the incoming flow at $n^{\mathrm{out}}_{a,k}$ is exactly $T_a$ if $N$ fires and $0$ otherwise. We connect $n^{\mathrm{out}}_{a,k}$ to $e_k$ with capacity $[1,1]$ (simulating the energy expenditure of one unit). We connect $n^{\mathrm{out}}_{a,k}$ to $n^{\mathrm{ass}}_{a,k}$ with capacity $[\mathrm{out}_a,\mathrm{out}_a]$, $\mathrm{out}_a = \sum_{b: s(a,b) \in S} w(a,b)$, and connect $n^{\mathrm{ass}}_{a,k}$ with the post-synaptic neurons $n_{b,k+d}$ with capacity $[w(a,b),w(a,b)]$. Let $\mathrm{res} = T_a + 1 - \sum_{b: s(a,b) \in S} w(a,b)$ be the difference between the threshold potential of $a$ and the sum of its weighted outgoing connections plus $1$. We distinguish between the three cases where $\mathrm{res} < 0$, $\mathrm{res} > 0$, and $\mathrm{res} = 0$.
\begin{itemize}
\item[$<$] We introduce an auxiliary source $p_{a,k}$ and connect $p_{a,k}$ to $n^{\mathrm{out}}_{a,k}$ with capacity $[-\mathrm{res},-\mathrm{res}]$; we have that $\mathrm{out}_a = T_a - \mathrm{res} - 1$ (Figure \ref{synapse_gadget}).
\item[$>$] We introduce an auxiliary sink $r_{a,k}$ and connect $n^{\mathrm{out}}_{a,k}$ to $r_{a,k}$ to with capacity $[\mathrm{res},\mathrm{res}]$; we have that $\mathrm{out}_a = T_a + \mathrm{res} - 1$.
\item[$=$] We do not introduce additional nodes and have that $\mathrm{out}_a = T_a - 1$.
\end{itemize}
\end{enumerate}

\begin{figure}[h!]
\centering
\includegraphics[width=9.8cm]{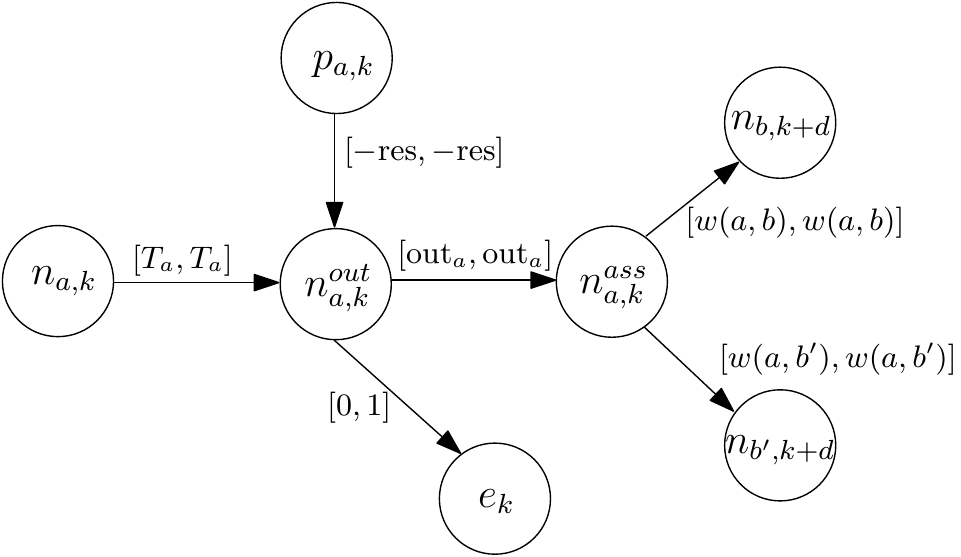}
\caption{Synapse gadget $G_a$ distributing flow over the post-synaptic neurons. Any excess flow above the potential $T_a$ is directed to an auxilliary sink $R_{T_a,k}$. This specific graph describes the case where $\mathrm{res} < 0$, that is, the weighted outgoing spikes exceed the threshold potential.}
\label{synapse_gadget}
\rule{\columnwidth}{0.3mm}
\vspace{1mm}
\end{figure}

\begin{enumerate}
\item[7.] We introduce a failure gadget $F$ as follows. We introduce vertices $f_l, 1 \leq l \leq t$, with a source $s_f$ and a sink $t_f$, connect $s_f$ to $f_1$, $f_l$ to $f_{l+1} (l < t)$ and $f_l$ to $t_f$. We connect every vertex $n_{i,k}$ to $f_k$ and define the capacity between each vertex in this gadget to be $[0,1]$. The intuition here is that we can push flow $1$ from $s_f$ to $t_f$ if and only if there is no flow at all from any node $n_{i,k}$ to $f_k$. In this way we will enforce desired behaviour when a threshold in reached at a neuron $a$: it {\em must} send its total flow $T_a$ to $n^{\mathrm{out}}_{a,k}$ since otherwise it is forced to send one unit of flow to the failure gadget.
\end{enumerate}

\begin{enumerate}
\item[8.] Finally, we introduce a `master' source $s$ and sink $t$ and connect $s$ to $s_t$, $s_f$, and $s_e$ (with capacity $[0,1]$ for both) and connect $t_t$, $t_f$, and $t_e$ to $t$, again both with capacity $[0,1]$.
\end{enumerate}

Given this construct, we claim that we can send a flow $f = 3$ from $s$ to $t$ if and only if the spiking neural network accepts within time and energy limits $t$ and $e$. We are now ready to formally prove the following theorem:

\begin{theorem}
\TNwFlow\ is \\$\mathcal{S}(\mathcal{O}(n),\mathcal{O}(n)),(\mathcal{O}(1),\mathcal{O}(n),\mathcal{O}(n))$-hard.
\end{theorem}

\begin{proof}
To prove hardness we must show that every constant-time-bounded computation of a spiking neural network $\mathcal{S}$ with $\mathcal{O}(n)$ neurons, satisfying the constraints addressed above, can be encoded by an instance $(G, d)$ (with $\mathcal{O}(n)$ nodes) of \TNwFlow\ such that $(G, d)$ is a yes-instance of \TNwFlow\ if and only if $\mathcal{S}$ accepts. Let $c_t$ be a constant bound on the runtime of $\mathcal{S}$, enforced by its clock, and let $c_e$ be a linear bound on the energy usage of $\mathcal{S}$, enforced by its meter. Observe that the construction of $(G, d)$ from $\mathcal{S}$ ensures the following behaviour:
\begin{enumerate}
\item All constant inputs must fire for $\mathcal{S}$ to accept, otherwise no threshold at all will be reached. That is, we may assume that the amount of flow entering the network (apart from the auxilliary sinks and sources in the synapse gadgets whose purpose is to keep the flow locally invariant) is fixed.
\item The flow going from $n_{a,k}$ to $n_{a,k+1}$ is exactly the increase in network potential of $N_a$ between time $k$ and time $k+1$;
\item There can be a flow of exactly $T_a$ going from $n_{a,k}$ to $n^{\mathrm{out}}_{a,k}$ if and only if the threshold $T_a$ is reached at time $k$, allowing $N_a$ to fire at time $k$;
\item If the threshold $T_a$ is reached at time $k$, $a$ {\em must} fire, that is, send a flow of size $T_a$ to $n^{\mathrm{out}}_{a,k}$, otherwise by construction it will send one unit of flow to the failure gadget;
\item Let $N_b$ be a post-synaptic neuron with respect to $N_a$ in $\mathcal{S}$ and let $d$ and $w$ denote its delay and weight. The flow from $n^{\mathrm{out}}_{a,k}$ to $n_{b,k+d}$ is exactly $w(a,b)$ if $N_a$ fires at time $k$, and zero otherwise;
\item We can push a flow of $1$ from $s$, via $s_e$ and $t_e$, to $t$ if and only if this channel is not `saturated' by more than $e-1$ neurons `firing';
\item We can push a flow of $1$ from $s$, via $s_t$ and $t_t$, to $t$ if and only if this channel is not `saturated' by $n_{\mathrm{rej}}$ `firing' until at least time $t-1$.
\end{enumerate}
Now, assume that $\mathcal{S}$ accepts using at most energy $e$ and fires before time $t$. Then, the construction above ensures that there is a path of flow in the nodes $n$ in the flow network such that $n_{\mathrm{rej}}$ will stop firing in time; it also ensures that there will be at most $e-1$ units of flow entering the channel from $s_e$ to $t_e$, and that no flow enters the failure gadget, that is, we can push a flow $f = 3$ from $s$ to $t$. If $\mathcal{S}$ does not accept (or not in time, or with too much energy) either the time or energy gadget will be saturated. Any flow that violates the firing characteristics of the neurons will cause the failure gadget to be saturated. We conclude that $(G, d)$ is a yes-instance of \TNwFlow\ if and only if $\mathcal{S}$ accepts; note that only a constant number of vertices is introduced in $G$ for every neuron in $\mathcal{S}$.
\end{proof}

\bibliographystyle{unsrt}  
\bibliography{snn_network_flow}

\end{document}